\documentclass[10pt,twocolumn]{article}
\usepackage[letterpaper,margin=0.9in,columnsep=0.28in]{geometry}
\usepackage[T1]{fontenc}
\usepackage[utf8]{inputenc}
\usepackage{amsmath,amssymb,amsthm}
\usepackage{booktabs}
\usepackage{multirow}
\usepackage{graphicx}
\usepackage[numbers,sort&compress]{natbib}
\usepackage[hidelinks]{hyperref}
\usepackage{microtype}

\newcommand{\tabmathcaption}{Extinct supports hold 6.1\% of tokens (9.0\% of rows). At the survivor-conditional interior budget $B=1{,}989$ tokens ($M_{\mathrm{safe}}=0.08$, tolerance $0.5\%$), the survivor arms assemble with $G=+9$ while $D_{00}^{\mathrm{all}}$ has $G=-3$.}

\newtheorem{proposition}{Proposition}
\newtheorem{lemma}{Lemma}
\newcommand{\Fset}{\mathcal{F}}

\title{\vspace{-1.2em}Counterfactual Gate--Support Decomposition:\\
Separating Selection from Redistribution in Training-Data Filtering}
\author{Hyojung Han$^{\dagger}$ \qquad Jongmin Kim \qquad Seung-Hun Jeon\\[4pt]
        \small ThakiCloud\\[2pt]
        \small $^{\dagger}$Corresponding author: \texttt{hyojung.han@thakicloud.com}}
\date{2026-10-06}

\begin{document}
\maketitle

\begin{abstract}
A training-data quality gate changes both which examples survive within regions of the input
space and how training mass is distributed across those regions. A gate-versus-random
comparison therefore does not isolate these mechanisms. We introduce Counterfactual
Gate--Support Decomposition (CGSD), a four-arm, budget-matched design that crosses gate
eligibility with support distribution and decomposes the observed gate effect into
within-support selection, between-support redistribution, and their interaction. The practical
constraint is finite-pool realizability: we characterize mass and token-assembly conditions
under which the required arms can be constructed, together with a support-extinction case in
which the ordinary design is undefined. On Korean--English and Korean--Japanese retrieval,
the mean within-support selection component is the only one that exceeds its calibrated null
band ($+29.7$ and $+5.6$ holdout items; $30/30$ positive replicates), decisively in
Korean--English and marginally in Korean--Japanese; redistribution and interaction remain
unresolved at the available power. The same pattern holds for a second, unrelated gate family --
a model-confidence threshold rather than a deterministic verifier -- judged against its own null
($+27.3$ items, $9/9$ positive), so the result is not a property of one gate's mechanism. On a mathematics gate, zero-eligibility supports leave the
ordinary design undefined at every budget, while the extended design is constructible only
inside narrow budget bands that an exhaustive scan locates and that the natural operating point
misses --- the assembly condition separating a property of the pool from a limitation of the
packer. CGSD therefore returns a decomposition, a constructible-but-underpowered design, or an
explicit non-constructibility result, and says which; it prevents support redistribution from
being silently interpreted as sample-level gate quality.
\end{abstract}

\section{Introduction}

Data-quality gates are deployed on the premise that removing bad examples improves the model
trained on what remains. The evidence usually offered is a comparison between a gated training
set and a random subset of the same size. That comparison answers a different question than the
one being asked, because a gate does not only choose better examples --- it also reweights the
distribution over whatever regions the data occupies. A gate that removes 40\% of a corpus may
be removing 5\% of one topic and 95\% of another. The downstream difference then contains a
selection effect and a redistribution effect, and the comparison cannot say which one it saw.

This matters in practice because the two point at different levers. If the resolved gain is
within-support selection, improving the within-support selection rule is the relevant
intervention. If redistribution dominates, the gate is functioning as an accidental curriculum
and explicit mixture control is the more direct and more controllable alternative. If the two
have opposite signs, tightening the gate can make things worse in a way that looks inexplicable
from the aggregate number alone.

We propose to answer the question by construction. Fix a source pool, a support representation
defined on \emph{pre-gate} text, and a training token budget. Build four training sets that
cross two binary factors --- eligibility (unrestricted vs.\ gate-eligible) and support
distribution (source-preserved vs.\ gate-induced) --- at the same budget. The observed gate
effect then decomposes exactly, replicate by replicate, into three terms. No modelling
assumption is introduced; the arms differ only in the factor being varied.

The algebra here is the familiar $2\times2$ factorial identity, and we do not claim it as new.
Our contribution is making that contrast operational for finite training pools under fixed
exposure, and characterizing when the arms the identity requires cannot be realized at all.

The price of that exactness is that the arms have to exist. Our main methodological finding is
that existence is not a technicality. It fails in four ways whose prescriptions point in
different directions: a support cannot supply its share, which a smaller budget fixes; a support
has no eligible mass at all, which no budget fixes; the available example sizes cannot tile the
per-support quotas at this budget, which is a property of the pool and not monotone in the
budget; and an assembly exists but the packer does not find it, which is a property of the
algorithm alone. Half of this paper is about telling them apart, because reading one as another
sends you to change the wrong thing --- as it did us, on our own mathematics pool.

\paragraph{Contributions.}
(i) A construction-based operationalization of the $2\times2$ gate contrast for finite training
pools at fixed token exposure, separating within-support selection from between-support
redistribution. (ii) Two independent feasibility conditions --- mass and assembly --- and their
joint feasible region as a function of budget and support partition, with a bound linking the
budget tolerance to the realized support distribution. (iii) An extension for \emph{support
extinction}, where the ordinary design is undefined at every positive budget, isolating the
extensive margin as a signed contrast. (iv) A validation protocol --- signature-locked null
calibration, analytic known-truth recovery with interval coverage, infeasible controls, and
extinction controls --- under which we then demonstrate the method on two retrieval locales and
two unrelated gate families (a deterministic verifier and a model-confidence threshold), and
a feasibility audit of mathematics and text-to-SQL pools that returns extinction, empty feasible
regions, and one narrow feasible band, each distinguished from the others rather than reported
as a single failure.

\section{Related work}

\paragraph{Data filtering and selection.}
A large body of work asks which examples to keep. Heuristic and model-based quality filters,
deduplication~\citep{lee2022deduplicating,abbas2023semdedup}, difficulty- or margin-based
pruning~\citep{sorscher2022beyond}, influence-guided instance selection for a target
task~\citep{xia2024less}, and datamodel-driven selection~\citep{engstrom2024dsdm} all improve
downstream performance by discarding part of a pool. Controlled curation benchmarks make the
comparison itself the object of study~\citep{gadre2023datacomp}. These methods are prescriptive:
they propose a better rule. CGSD is diagnostic and takes the rule as given.

\paragraph{Data mixture and domain reweighting.}
A parallel line shows that the proportions themselves change downstream performance, independent
of which individual examples are chosen: reweighting domains with a proxy
model~\citep{xie2023doremi}, generalization-estimated reweighting~\citep{fan2024doge},
regression- and law-based mixture search~\citep{liu2024regmix,ye2024datamixinglaws}, importance
resampling toward a target distribution~\citep{xie2023dsir}, and unified treatments of the mixing
objective~\citep{chen2024aioli}. This literature establishes that mixture is a real and separate
lever, which is precisely why a filter that silently moves the mixture is hard to interpret.
CGSD is the measurement that separates the two levers for a filter that was not designed to move
either one deliberately.

\paragraph{Attributing performance to training data.}
Influence functions~\citep{koh2017influence,grosse2023studying}, gradient-trace
attribution~\citep{pruthi2020tracin}, datamodels~\citep{ilyas2022datamodels}, and Shapley-style
valuation~\citep{ghorbani2019datashapley,wang2025inrunshapley} attribute performance to
individual training examples, usually without retraining. CGSD works at a coarser grain and in
the opposite direction: it retrains on deliberately constructed pools to attribute an effect to
two \emph{mechanisms} rather than to examples, which is why it can be run on a gate whose scoring
function is opaque.

\paragraph{Decomposition, overlap and interference.}
Splitting an aggregate gap into a component due to differing composition and a component due to
differing returns is the Oaxaca--Blinder decomposition~\citep{oaxaca1973male,blinder1973wage},
and splitting a total effect into direct and indirect parts is the mediation
literature~\citep{pearl2001direct}. Our mass condition is the finite-sample analogue of the
positivity requirement that underlies matching and propensity
methods~\citep{rosenbaum1983propensity,stuart2010matching}: a contrast that no unit can supply is
not estimable at any sample size. Because one model is trained per arm on a pooled corpus, we
also inherit the interference concerns formalized for causal inference under
spillover~\citep{hudgens2008toward}; we return to this in the limitations.

\paragraph{Experimental rigour.}
Seed variance, under-reported variability and fragile comparisons are well documented in
ML~\citep{henderson2018deep,picard2021torch,dodge2019show}, as are the reporting practices and
interval procedures that mitigate
them~\citep{pineau2021improving,bouthillier2021accounting,agarwal2021deep}. Our validation
protocol (Section~\ref{sec:validation}) follows that tradition and adds one element specific to
this setting: a null band that is bound to the exact configuration it was measured on, so that a
threshold cannot silently be reused across operating points.

\section{The estimand}
\label{sec:estimand}

Let $D=\{x_i\}$ be a source pool. Each example carries a token count $t_i$, a gate verdict
$g_i\in\{0,1\}$, and a support label $z_i=\Pi(x_i)$ produced by a partition $\Pi$ of the
pre-gate input. We write $T=\sum_i t_i$ and, for support $c$,
$T_c=\sum_{i:z_i=c} t_i$, $E_c=\sum_{i:z_i=c} g_i t_i$, and $e_c=E_c/T_c$ for the
\emph{eligibility rate}.

\paragraph{Support must be pre-gate.}
$\Pi$ is computed from the input text only. If the support label were derived from anything the
gate produces, the redistribution term would be defined in terms of the treatment and the
decomposition would be circular. This is not hypothetical: in our own pipeline the most
convenient available label was a gate output, and using it would have made the gate's failure
reason its own support.

\paragraph{Four arms.}
Given a token budget $B$, define the source share $p_c=T_c/T$ and the gate-induced share
$p^G_c=E_c/\sum_{c'}E_{c'}$. Each arm fills per-support budgets by sampling to a token target:

\begin{center}
\begin{tabular}{lll}
\toprule
arm & eligibility & support share \\
\midrule
$D_{00}$ & unrestricted & $p_c$ \\
$D_{10}$ & gate-eligible & $p_c$ \\
$D_{01}$ & unrestricted & $p^G_c$ \\
$D_{11}$ & gate-eligible & $p^G_c$ \\
\midrule
$R$ & unrestricted & none (budget only) \\
\bottomrule
\end{tabular}
\end{center}

The decomposition uses the four support-controlled arms. $R$ is a fifth, literature-facing
reference: it spends the same token budget with no support control at all, which is what a
plain ``gate versus random subset'' comparison implicitly uses as its baseline. We report
$Y(D_{00})-Y(R)$ so that the effect of imposing support control is visible separately from the
gate contrast itself.

Writing $Y(\cdot)$ for the downstream metric of a model trained on an arm,
\begin{align}
E &= Y(D_{10}) - Y(D_{00}), \label{eq:E}\\
S &= Y(D_{01}) - Y(D_{00}), \label{eq:S}\\
I &= Y(D_{11}) - Y(D_{10}) - Y(D_{01}) + Y(D_{00}), \label{eq:I}
\end{align}
and the identity $\Delta_G = Y(D_{11})-Y(D_{00}) = E+S+I$ holds exactly for every replicate,
by construction rather than in expectation.

\paragraph{What $E$ is not.}
$E$ is the effect of restricting to gate-eligible examples \emph{while holding the support
distribution fixed}. It is not ``quality''. It does not say why eligibility helps --- label
correctness, length, formatting, and annotator agreement all live inside it. We keep the
neutral name throughout because the temptation to read $E$ as a quality measurement is the
single easiest way to overclaim from this decomposition.

\paragraph{The estimand is budget- and partition-conditional.}
$E$, $S$ and $I$ are defined at a particular $B$ and $\Pi$. Two operating points give two
estimands, so absolute magnitudes are not comparable across pools with different budgets or
holdout sizes. We write $E(B,\Pi)$ when the dependence matters.

\section{Feasibility: two independent conditions}

\subsection{Mass}
$D_{10}$ keeps the source share inside the eligible pool, so it needs
$B\,p_c \le E_c$ for every $c$, i.e.\ $B/T \le \min_c e_c$. Define
\begin{equation}
M(B,\Pi) \;=\; \min_c e_c \;-\; B/T .
\label{eq:M}
\end{equation}
$M<0$ is a positivity violation at this budget and is repaired by lowering $B$.

\subsection{Assembly}
Examples are indivisible, and the four arms do not draw from the same pool: $D_{00}$ and
$D_{01}$ draw from the unrestricted pool while $D_{10}$ and $D_{11}$ draw only from the
gate-eligible one, at two different target shares. Assembly must therefore be checked per arm.
Index eligibility by $a\in\{0,1\}$ and support share by $s\in\{0,1\}$, and let
$\mathcal{A}_{a,c}$ be the set of token sums attainable in support $c$ from the examples
eligibility level $a$ admits. With $p^{(0)}_c=p_c$, $p^{(1)}_c=p^G_c$ and
$b^{(s)}_c=B\,p^{(s)}_c$,
\begin{align}
Q_{as}(B) &= \sum_c \Bigl( b^{(s)}_c - \max\{x\in\mathcal{A}_{a,c} : x\le b^{(s)}_c\} \Bigr),
\label{eq:Q}\\
G_{as}(B,\Pi;\tau) &= \tau B - Q_{as}(B), \qquad
G_{\min} = \min_{a,s} G_{as},
\label{eq:G}
\end{align}
where $\tau$ is the budget matching tolerance. The constructor never overshoots a per-support
target --- overshooting would break the support match --- so each $Q_{as}$ is a sum of
non-negative shortfalls. $G_{as}\ge 0$ says an in-tolerance assembly of that arm
\emph{exists}. It is a property of the pool under the stated matching contract, not of the
packing heuristic: a greedy packer may fail where $G_{as}\ge0$, and that is a separate,
algorithmic problem. Reporting one as the other leads directly to requesting data that is not
needed. We compute $\mathcal{A}_{a,c}$ exactly by a bitset subset-sum and fall back to the
packer only when that is prohibitive, flagging the fallback; every feasibility verdict reported
in this paper is exact, and no negative verdict rests on a heuristic packer failure.

Crucially $G_{as}$ is not monotone in $B$: it is a lattice quantity, and moving the budget
slightly can flip its sign. Lowering $B$ does not reliably repair assembly, and in general makes
it worse, since $\tau B$ shrinks while $Q_{as}$ does not.

A tolerance on the \emph{total} budget would be a weak guarantee if it allowed the realized
support distribution to drift, so we record what it actually buys.

\begin{lemma}[Support drift under no-overshoot]
\label{lem:tv}
Let $\tilde p$ be the realized support distribution of an arm whose per-support shortfalls are
non-negative and whose total shortfall is at most $\tau B$. Then
$\operatorname{TV}(\tilde p, p^{(s)}) \le \tau/(1-\tau)$.
\end{lemma}

\begin{proof}
Write the realized support totals as $\tilde b_c = b^{(s)}_c - q_c$ with $q_c\ge0$ and
$\sum_c q_c = q \le \tau B$, and let $\tilde B = B - q$. Then
$\tilde p_c - p^{(s)}_c = (b^{(s)}_c - q_c)/\tilde B - b^{(s)}_c/B$, so
$\operatorname{TV} = \tfrac12\sum_c |\,\tilde p_c - p^{(s)}_c| \le \tfrac12\bigl(
\sum_c q_c/\tilde B + \sum_c b^{(s)}_c\,(1/\tilde B - 1/B)\bigr) = q/\tilde B \le
\tau B/(B-\tau B) = \tau/(1-\tau)$.
\end{proof}

At $\tau=0.5\%$ this bounds realized support drift at $0.503\%$ in total variation, so the
budget tolerance and the support match are not two independent promises --- the first implies
the second.

\subsection{The feasible region}
\begin{equation}
\Fset_\Pi = \{\,B>0 : M(B,\Pi)\ge 0 \ \wedge\ G_{\min}(B,\Pi;\tau)\ge 0\,\}.
\label{eq:F}
\end{equation}
Both conditions depend on the partition, so feasibility is joint in exposure and support
representation. Figure~\ref{fig:region} shows $\Fset_\Pi$ shrinking as the partition is refined
on our retrieval pool. We deliberately write these as functions of the partition rather than of $K$: two partitions
with the same cardinality can have very different granularity.

\begin{figure}[tb]
\centering
\includegraphics[width=\columnwidth]{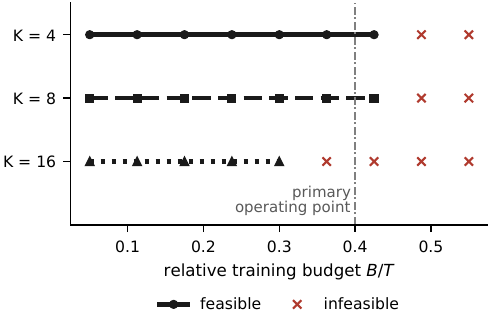}
\caption{The feasible region $\Fset_\Pi$ of \eqref{eq:F} over relative training budget, for
three $k$-means partitions of the Korean--English pool at $\tau=0.5\%$; ``feasible'' marks a
budget in $\Fset_\Pi$. Refining the partition
from $K{=}8$ to $K{=}16$ moves the primary operating point $B/T{=}0.40$ outside the feasible
region. For this pool the binding condition is $M$ throughout: the pool is large enough that
$G$ stays comfortably positive ($+51$ to $+504$ tokens of slack). $G$ binds instead on the small
gated pools of Section~\ref{sec:math}.}
\label{fig:region}
\end{figure}

\paragraph{The realization contract.}
It is worth separating two things that are easy to conflate. The target budget and the target
support distribution define the \emph{estimand}. The tolerance $\tau$, the no-overshoot rule
and the no-cross-support-compensation rule define the \emph{finite-pool realization contract}
under which we regard a constructed arm as an admissible realization of that estimand. All
three could be relaxed to enlarge $\Fset_\Pi$, and each relaxation would make an infeasible
domain look feasible.

We fixed $\tau=0.5\%$ before any domain-level feasibility was inspected; under the no-overshoot
rule Lemma~\ref{lem:tv} shows this also bounds realized support drift at $0.503\%$, which is why
one number can govern both matched quantities. Relaxing the contract after observing
infeasibility does not repair the original audit --- it defines a different approximation to the
target counterfactual, and reports that approximation under the original name. A sensitivity
analysis at $\tau=1\%$ is legitimate and we would report it as such; using it to rescue a domain
that is infeasible at the pre-specified contract is not.

\section{Support extinction}
\label{sec:extinction}

\begin{proposition}[Structural extinction boundary]
\label{prop:ext}
If some source support has positive mass and zero gate eligibility, the support-preserving
gate-eligible counterfactual is infeasible for every positive training budget.
\end{proposition}

\begin{proof}
$D_{10}$ requires $B p_c \le E_c$ in support $c$. If $E_c=0$ and $p_c>0$, then any $B>0$
violates it.
\end{proof}

This is categorically different from $M<0$, which a smaller budget repairs. Let
$A=\{c: e_c>0\}$ and $Z=\{c: e_c=0\}$. Define $D_{00}^{\mathrm{all}}$ as the source-preserving
arm over all supports and $D_{00}^{A}$ as the arm that drops $Z$ and renormalises the
survivors' original relative proportions, both at the same budget $B$. With
$E_A,S_A,I_A$ defined as in \eqref{eq:E}--\eqref{eq:I} but over $A$,
\begin{align}
\Delta_G &= X_B + E_A + S_A + I_A, \label{eq:X}\\
X_B &= Y(D_{00}^{A}) - Y(D_{00}^{\mathrm{all}}), \nonumber
\end{align}
and the four terms telescope to $Y(D_{11})-Y(D_{00}^{\mathrm{all}})$ exactly. No separate
$D_{11}^{A}$ is needed: the real gate arm already places zero mass on $Z$.

$X_B$ is the \emph{support-extinction contrast}, or extensive-margin support effect. It is not
``the value of the tag'': $D_{00}^{A}$ does not merely discard the extinct mass, it returns
that mass to the survivors in their original proportions, so $X_B$ carries a replacement policy
with it. A different replacement policy is a different $X_B$. If a cost is wanted, it is $-X_B$.

Reporting $E_A,S_A,I_A$ without $X_B$ silently changes the estimand from $\Delta_G$ to
$\Delta_{G\mid A}$, and deletes exactly the phenomenon that motivated looking at the domain. We
treat $X_B$ as mandatory whenever $Z\neq\emptyset$.

\section{Validation before use}
\label{sec:validation}

The decomposition is only as good as the constructor and the estimator, so we validate both
before reading any result, in layers that each fail for a different reason.

\textbf{Invariants.} Arms land within $\tau$ of $B$; support histograms match their targets; no
holdout example or holdout query appears in any arm.

\textbf{Null calibration.} With eligibility replaced by a random mask of the same rate, the
decomposition should return nothing. Over 200 replicates we take
$\delta=q_{95}(|\text{null}|)$ per component, with a distribution-free order-statistic interval
for that quantile. The null median is $0$ items in both locales and the sign tests are not
significant ($p=0.94/0.88/0.38$ for EN, $0.27/0.38/0.52$ for JP).

We state the decision rule explicitly, because it is a threshold and not a level-$\alpha$ test:
a component is called \emph{resolved} when $|\hat\theta| > \delta$, where $\hat\theta$ is the
mean over replicates. Two caveats follow and we do not hide either. First, each null replicate
is a single selection draw summarized as a median over training seeds, whereas $\hat\theta$ is a
mean over $30$ nested replicates, so the two are not the same statistic. We do not leave the
direction to argument, and we do not argue it from variance alone, since a larger variance does
not by itself dominate a quantile. Aggregating the stored null replicates to the same level as
$\hat\theta$ --- averaging ten single-draw replicates, matching the ten selection seeds --- and
recomputing the same $95$th percentile of the absolute value gives $4.6$, $2.0$ and $2.5$ items
for $E$, $S$ and $I$, against the $\delta$ we actually use of $14.0$, $6.0$ and $7.0$. The
threshold we apply is thus $2.8$--$3.0\times$ the one a matched aggregation would produce: the
rule is \emph{conservative}, not anticonservative, and $\delta$ is correspondingly not an
estimate of the standard error of $\hat\theta$. We keep the pre-registered single-draw
$\delta$ rather than switching to the matched one, because choosing the smaller threshold after
seeing which components it would resolve is exactly the selective judgement the protocol exists
to prevent. Second, $\delta$ is itself estimated and carries an order-statistic interval,
which we report alongside it; a component whose point estimate sits inside that interval should
be read as marginal. To make that reading concrete without inventing a second test, we report a
\emph{conservative interval-separation diagnostic}: whether the effect's own interval clears the
\emph{upper} end of $\delta$'s interval,
$\mathrm{CI}_{\mathrm{lower}} > \delta_{\mathrm{upper}}$. This charges the threshold's estimation
error against the effect rather than ignoring it. It is a diagnostic label, not a hypothesis
test, and it never overrides the pre-specified rule --- it only says how much room that verdict
had. We apply no multiplicity correction across the six primary component tests
and therefore do not attach a family-wise error rate to the pattern as a whole.

\textbf{Known-effect recovery (analytic).} On an outcome that is a known function of the arm
composition, the model-implied $E^\ast,S^\ast,I^\ast$ can be computed in closed form, so the
constructor, the decomposition algebra and the interval procedure can be checked against truth
without involving a learner. Over 300 replicates in each of ten settings per locale --- four planted
conditions ($E$ only, $S$ only, both, and one with a planted interaction) at an interior margin,
plus two of those conditions re-run at each of three frontier margins --- the bias was
indistinguishable from zero ($|\mathrm{bias}|\le 0.18$ items, within $0.6$ standard errors) and
the empirical coverage of the nominal 95\% intervals was $0.927$--$0.960$ over all 60
component-settings (Monte-Carlo error $\pm0.01$). An eleventh setting per locale is the
infeasible control below, which raises rather than reports. This is the check that the hierarchical bootstrap is not anticonservative; it is
performed where ground truth exists, because it cannot be performed on the trained runs.

\textbf{Known-effect recovery (trained).} Synthetic interventions that plant an effect in a
known component are recovered in the right component with the right sign, now with the learner
in the loop.

\textbf{Frontier stress.} Shrinking the mass margin from $M=0.07$ to $M=0.01$ at fixed budget
left bias, RMSE and coverage unchanged, so proximity to the mass boundary does not by itself
degrade the estimator. It does shrink what there is to estimate: at the tightest margin the
constructible support effect was under one holdout item.

\paragraph{Calibration signatures.}
A null band is only valid for the configuration it was measured on. We bind each $\delta$ to a
signature over locale, partition cardinality, budget, holdout size, training seeds, the gate's
identity and its realized pass rate, endpoint, and the constructor's source hash, and refuse to attach a verdict unless the signature matches
exactly. This check first fired on our own previously reported results: a $\delta$ calibrated at
one budget and seed count was being applied to runs at another. Those labels were withdrawn and
both locales were recalibrated at their operating points, which is the calibration reported
here. We also require every signature input to be observed or explicitly supplied; a default
that fills in a missing field makes the signature assert something nobody measured.

\begin{figure}[tb]
\centering
\includegraphics[width=\columnwidth]{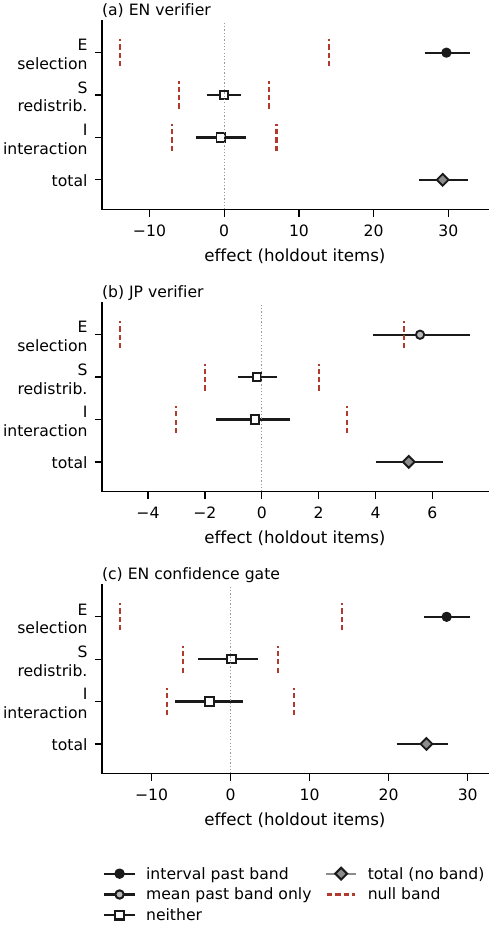}
\caption{Component means with 95\% hierarchical-bootstrap confidence intervals, in holdout
items, for (a) the Korean--English verifier gate, (b) the Korean--Japanese verifier gate, and
(c) EN\textsubscript{B}, a model-confidence gate on the Korean--English pool. Rows are the
within-support selection term $E$ of \eqref{eq:E}, the between-support redistribution term $S$
of \eqref{eq:S}, the interaction $I$ of \eqref{eq:I}, and the total $\Delta_G=E+S+I$. Vertical
dashed marks are $\pm\delta$, the null band calibrated from 200 placebo replicates \emph{at that
locale's own operating point}. A solid circle means the
effect interval clears the upper end of $\delta$'s own interval; a grey circle means only the
point estimate clears $\delta$, which is the pre-specified rule and the marginal case; an open
square means neither. The total has no calibrated band and is drawn
with a distinct marker. Panels are separate because the holdout
sizes differ (3{,}914 vs.\ 1{,}465 items), so the two item counts are not comparable quantities.}
\label{fig:decomposition}
\end{figure}

\section{Retrieval: the effect is within-support}

The gate maps Korean queries onto a product taxonomy leaf and accepts a pair when both a mapping
check and a translation check pass. Support is $k$-means on frozen embeddings of the pre-gate
query, $K=8$. The downstream model is a static-embedding retriever and the endpoint is top-1
leaf accuracy on a held-out split. Each locale runs 10 selection seeds $\times$ 3 training
seeds, and confidence intervals are hierarchical bootstraps over that nesting.

Table~\ref{tab:verdict} and Figure~\ref{fig:decomposition} give the result. Within-support
selection is the only component whose mean clears its calibrated null band in either locale;
redistribution and interaction remain unresolved at the available power. The two locales differ
in how much room that verdict had: EN $E$ also satisfies the conservative separation diagnostic
($\mathrm{CI}_{\mathrm{lower}}=+26.8$ against $\delta_{\mathrm{upper}}=17$), while JP $E$ does
not ($+3.9$ against $6$). Signs are $30/0/0$ for
$E$ and for $\Delta_G$, and the per-replicate identity residual is exactly zero. An item-level
three-stage bootstrap, which also resamples the evaluation set and is wider by construction,
still excludes zero for $E$: $[+16.1,+46.0]$ for EN and $[+2.0,+9.8]$ for JP.\footnote{The JP
interval comes from the same $10\times3$ run as the primary result. The EN item vectors were
stored only for a reduced $3\times3$ sub-run, so the EN item-level interval is computed on nine
replicates rather than thirty and is correspondingly less precise.}

Two qualifications belong with that sentence rather than after it. The Japanese margin is
narrow: $\hat\theta=+5.6$ against $\delta=5.0$, with $\delta$'s own order-statistic interval
$[4.0,6.0]$ containing the point estimate, and with the lower end of the effect's own interval
($+3.9$) below $\delta$. JP $E$ therefore passes the pre-specified point-estimate rule and is
marginal under the interval-based reading --- a third state, not a softer version of either. We
keep the pre-specified verdict rather than demoting it after seeing the result, and we decline to
present it as the same kind of detection as EN. The Japanese run is a replication of the
component \emph{pattern}, not an independently strong detection. And the point estimates place nearly all of the observed total on $E$ in both
locales, but that is a statement about point estimates: we do not infer $S=I=0$.

\begin{table*}[t]
\centering
\small
\setlength{\tabcolsep}{6pt}
\begin{tabular}{llrrcr@{\,}ll}
\toprule
& & \multicolumn{1}{c}{mean} & \multicolumn{1}{c}{95\% CI} & $n_+/n_-/n_0$ & \multicolumn{2}{c}{$\delta$ [CI]} & evidence \\
\midrule
\multirow{4}{*}{EN} & $E$ & $+29.7$ & $[+26.8,\,+32.9]$ & 30/0/0 & 14.0 & [13,\,17] & point $+$ interval \\
 & $S$ & $-0.0$ & $[-2.3,\,+2.2]$ & 13/15/2 & 6.0 & [5,\,8] & -- \\
 & $I$ & $-0.5$ & $[-3.8,\,+2.9]$ & 16/13/1 & 7.0 & [6,\,11] & -- \\
 & $\Delta_G$ & $+29.2$ & $[+26.1,\,+32.6]$ & 30/0/0 & \multicolumn{2}{c}{--} & -- \\
\midrule
\multirow{4}{*}{JP} & $E$ & $+5.6$ & $[+3.9,\,+7.3]$ & 30/0/0 & 5.0 & [4,\,6] & point only \\
 & $S$ & $-0.2$ & $[-0.8,\,+0.5]$ & 9/13/8 & 2.0 & [2,\,3] & -- \\
 & $I$ & $-0.2$ & $[-1.6,\,+1.0]$ & 13/13/4 & 3.0 & [2,\,3] & -- \\
 & $\Delta_G$ & $+5.2$ & $[+4.0,\,+6.4]$ & 30/0/0 & \multicolumn{2}{c}{--} & -- \\
\midrule
\multirow{4}{*}{EN\textsubscript{B}} & $E$ & $+27.3$ & $[+24.4,\,+30.3]$ & 9/0/0 & 14.0 & [13,\,16] & point $+$ interval \\
 & $S$ & $+0.1$ & $[-4.1,\,+3.4]$ & 5/3/1 & 6.0 & [5,\,8] & -- \\
 & $I$ & $-2.7$ & $[-7.0,\,+1.6]$ & 4/5/0 & 8.1 & [7,\,10] & -- \\
 & $\Delta_G$ & $+24.8$ & $[+21.0,\,+27.6]$ & 9/0/0 & \multicolumn{2}{c}{--} & -- \\
\bottomrule
\end{tabular}

\caption{Decomposition at each locale's primary operating point, in holdout items, over 30
replicates. $\delta$ is that locale's own null band --- the 95th percentile of $|$null$|$ over
200 placebo replicates at a matching calibration signature --- reported with its order-statistic
interval. The evidence column names three states rather than starring two, because a star in an
ML table reads as ``statistically significant'' and $\delta$ is not a level-$\alpha$ threshold.
\emph{point only} means the pre-specified rule $|\hat\theta|>\delta$ passes; \emph{point $+$
interval} additionally means the effect's interval clears the \emph{upper} end of $\delta$'s own
interval. EN $E$ reaches the second; JP $E$ reaches the first and is marginal, its point estimate
falling inside $\delta$'s interval. A dash means the point estimate did not clear $\delta$ at
this power, which is not a statement that the component is zero. $\Delta_G$ is the total and has no calibrated band.}
\label{tab:verdict}
\end{table*}

\paragraph{A second gate family.}
Everything above audits one gate: a deterministic verifier whose eligibility is a map-and-\allowbreak
translation check. Whether CGSD says anything beyond that gate's mechanism is a question about the
method, so we ran it again on a gate that shares none of it --- a threshold on the generator's own
confidence score, eligible iff $\mathrm{conf}\ge0.75$. It is a different intervention in every
respect that matters here: it is model-based rather than rule-based, it is continuous before it is
thresholded, and it passes $60.3\%$ of the pool against the verifier's $47.9\%$. That difference in
exposure is why it cannot be judged against the verifier's band; it carries its own $200$-replicate
null, and the gate's identity and realized pass rate are inputs to the calibration signature so the
two can never be crossed by accident. The result is the third panel of Table~\ref{tab:verdict}:
$E=+27.3$ items with the interval clearing $\delta$'s upper end, $S$ and $I$ unresolved, and the
decomposition identity exact to the last digit. The same component carries the effect, on a gate
built on a different principle. We do not read across the two as an effect-size comparison ---
the pass rates differ and that difference is not controlled --- but the qualitative verdict is not
a property of how the first gate decides.

\paragraph{Reading.}
What the data support is that restricting to eligible examples helps within a fixed support
distribution. At the point estimates $E$ accounts for nearly all of the observed total, but the
data do not establish that redistribution is absent or unnecessary: $S$ and $I$ are individually
unresolved at this power, and no joint interval for $S+I$ was pre-specified. The qualitative component pattern replicates in Japanese:
$E$ alone crosses its pre-specified point-estimate band while $S$ and $I$ do not, and the JP $E$
result remains marginal under the interval-based reading. We make no cross-locale claim about effect
magnitude, because the operating points and holdout sizes differ. The Japanese run is a
replication of the same gate on the same data family, not an independent domain.

For readers who want the redistribution channel as a single number, the composite
$R = S+I = \Delta_G-E$ is $-0.5$ items for EN and $-0.4$ for JP by point estimate. We report it
as descriptive: it was not pre-specified, and it inherits the power limits of its parts.

\paragraph{Partition sensitivity.}
This is sensitivity, not robustness in the usual sense: the partition is part of the estimand,
so changing it changes what is being estimated rather than perturbing a fixed target. Across
$K=4,8,16$ at a budget interior for all three partitions, the qualitative conclusion that the
resolved component is $E$ persists, although its magnitude changes. The selection term stays
positive
($+16.0$, $+16.1$, $+9.8$ items) with no shift into redistribution. A paired contrast on the
same replicates gives $K{=}16$ minus $K{=}8$ of $-6.3$ items $[-12.4,-0.7]$ for $E$, an interval
excluding zero, with the difference appearing in $\Delta_G$ rather than in $S$. A finer
partition thus yields a materially smaller selection estimate, which is expected under a
partition-indexed estimand and is a reason to report $K$ with every number rather than to treat
any one value as canonical. These sensitivity runs use three selection seeds rather than ten and
are correspondingly noisier; with only three groups the between-seed variance component is not
usefully estimable and we do not report it. Replacing the support-preserving baseline $D_{00}$
with the unconstrained reference $R$ moves the estimate by $-0.5$ items $[-5.1,+3.5]$, so the
choice of baseline does not drive the conclusion.

\section{Feasibility audits: extinction, narrow bands, and empty regions}
\label{sec:math}

We applied the same audit to a gate over generated mathematics problems, chosen precisely
because a preliminary analysis suggested the gate had removed whole problem types. The support
label is the generator's own pre-gate template tag, so it satisfies the pre-gate requirement.

Two of four supports have $e_c=0$ (Table~\ref{tab:math}). By Proposition~\ref{prop:ext} the
ordinary four-arm construction is undefined at every positive budget, which we confirm
empirically down to $B=0.1\%$ of $T$. The extinct mass is $6.1\%$ of tokens and $9.0\%$ of rows;
we report the token figure because exposure is matched in tokens, and the two differ by a factor
of $1.5$ here.

Extending to \eqref{eq:X} makes the arms definable. At the pre-registered operating point it does
not yet make them constructible, and the reason turns out to be instructive. At the
survivor-conditional interior budget, the survivor arms assemble with $9$ tokens of slack, but
$D_{00}^{\mathrm{all}}$ has $G=-3$: no assembly of these examples lands within tolerance, and no
packer can change that, since the subset-sum here is solved exactly. The binding mismatch is
concentrated in the extinct supports, which supply $11.9$ of the $13.0$ tokens of $Q$, or 92\%.
In this pool the extinct supports are also the smallest supports, so their token lattices are
too coarse to satisfy the extensive-margin match. We do not claim this coincidence in general: a
large support can also be wiped out entirely by a gate, and would not produce the same assembly
failure.

The same audit applied to five gated text-to-SQL pools from the same generator finds no
extinction but a harsher assembly picture, and --- after the budget scan above --- one
exception. The largest pool ($578$ rows) has $\min_c e_c = 0.070$, which is below the interior
margin we fixed in advance; we report that as having no \emph{power-safe} operating point rather
than an empty $\mathcal{F}_\Pi$, since the margin is a power criterion and \eqref{eq:F} only
asks $M\ge0$. Scanning admissible budgets under $M\ge0$ on a uniform grid ($2{,}795$ budgets), four of the five
pools have an empty feasible region: not one budget gives $G\ge0$ for every arm. The fifth,
\textsc{live-clean} ($99$ rows), is constructible. Of the $86$ grid budgets that satisfy the
assembly condition we attempted all $86$, and the arms build at $68$, spanning
$B/T\in[0.309,0.643]$. The $18$ that fail all fail the same way, and the reason is worth stating
because it bounds what $G$ promises: $G\ge0$ certifies that an in-tolerance assembly exists, our
assembler solves the subset-sum exactly \emph{within each support}, and the tolerance is checked
on the \emph{total}. Per-support exactness does not compose into a global one, so a budget can
clear the condition and still defeat the solver. That gap --- $79\%$ of certified budgets
realized --- is a property of the algorithm, reported separately from the pool's own limit. SQL rows average $\sim\!155$ tokens
against $\sim\!39$ for the mathematics pool, so the same tolerance buys four times less assembly
freedom, which is why the feasible band is narrow where it exists at all. Cross-domain
decomposition therefore waits on data, not on further method development, and we can say by how
much. A single retrieval arm is trained on $95{,}162$ tokens; the entire mathematics pool is
$3{,}915$ tokens and the entire \textsc{live-clean} pool is $15{,}042$, so one arm of the
retrieval design consumes $24\times$ and $6\times$ their respective totals. At the feasible
budget a mathematics arm sees roughly $49$ rows. We could build these arms --- we have --- and
report a decomposition from them; we do not, because a component estimate at that scale would
carry an interval wide enough to contain any conclusion, and publishing it would be the
failure this paper is about. The endpoint compounds the gap: answer correctness over a
generative model is not the retrieval endpoint the estimator here is calibrated for, so a second
calibration would be required before any verdict could be attached. What these pools establish
is that the construction exists and that its feasibility conditions do the work they were
introduced for; what they cannot establish is an effect.

\begin{table}[tb]
\centering
\footnotesize
\setlength{\tabcolsep}{4pt}
\begin{tabular}{lrrrl}
\toprule
support & rows & tokens & $e_c$ & status \\
\midrule
\texttt{poly\_derivative} & 7 & 188 & 0.000 & \textbf{extinct} \\
\texttt{poly\_expand} & 44 & 1573 & 1.000 & surviving \\
\texttt{rational\_calc} & 2 & 50 & 0.000 & \textbf{extinct} \\
\texttt{series\_sum} & 47 & 2104 & 0.621 & surviving \\
\bottomrule
\end{tabular}

\caption{Support structure of the gated mathematics pool. Extinct supports have positive source
mass and zero gate-eligible mass, so Proposition~\ref{prop:ext} applies.
\tabmathcaption{}
At the pre-registered operating point the survivor-conditional arms assemble and the arm that
must reproduce the full source support does not. That failure is specific to this budget: see the
budget scan in the text.}
\label{tab:math}
\end{table}

It would be easy to stop here and report that the extended decomposition is not realizable on
this pool. That conclusion does not survive our own analysis. $G$ is a lattice quantity and is
not monotone in $B$, so a verdict read off one budget is a verdict about that budget. Scanning
every admissible integer budget exhaustively, $54$ of them satisfy the same pre-registered margin
\emph{and} give $G\ge0$ for all five arms, in a band just below the operating point
($B/T\in[0.480,0.495]$); the point we had fixed, $B=1989$, is the largest the margin allows and
falls in a gap between such bands. Moving $50$ tokens down to $B=1939$ flips
$D_{00}^{\mathrm{all}}$ from $G=-3$ to $G=+0.7$, and the five arms then build within tolerance
(worst deviation $0.46\%$ against $\tau=0.5\%$). Attempting all $54$ certified budgets together
with the pre-registered one, $49$ of $55$ build, and the six failures separate along exactly the
lines the taxonomy predicts: $B=1989$ reports \textsc{granularity} --- no assembly exists --- and
the other five report \textsc{packer} --- one exists and the solver does not find it. Recovering
the $49$ required replacing the greedy packer with an exact subset-sum assembler, randomized over
the witnesses that achieve the attainable sum so that selection within a support is not quietly
made a function of example length. The residual five are the same per-support-versus-global gap
described below: the assembler is exact inside each support, while $\tau$ is checked on the total.

What we therefore report is narrower and, we think, more useful than the claim we set out to
make: extinction makes the ordinary four-arm estimand undefined at every budget, and that part is
structural; the extended design is constructible on this pool, but only inside narrow budget
bands that the natural choice --- the largest budget the power margin allows --- misses. This is
the assembly condition doing exactly what it was introduced for, on us. The survivor-conditional terms
could be computed and would look like a result; substituting them for $\Delta_G$ would silently
replace the estimand with $\Delta_{G\mid A}$ and delete the extinction that motivated the domain,
which is the error this construction exists to prevent. We do not train on this pool.

\section{Limitations}

The retrieval evidence covers two gate families and two locales, but one backbone, one endpoint,
and one data family; it is a replication across gates and languages, not a demonstration of
domain generality. The second gate is also lower powered by design ($3$ selection seeds against
$10$), which is why its selection-variance component is estimated at zero: training noise
dominates at that number of groups, and the one-way estimator truncates. The cross-domain claim
that motivated this work --- that the same gate helps retrieval and hurts mathematics --- is
\emph{not} established here, because the mathematics side is a feasibility result rather than a
negative one. Unresolved components are statements about power, not about zero; the
Japanese null band is 2 items against a 1-item measurement grain. $E$ remains uninterpreted as
to mechanism. Finally, $\Fset_\Pi$ is defined relative to the matching contract; a different
contract is a different feasible region, and we regard changing it after seeing a failure as
out of bounds.

\section{Conclusion}

A gate-versus-random comparison does not, by itself, isolate within-support selection from
support redistribution. CGSD separates them by construction, and the price of the construction
is that the arms have to exist. Making existence explicit --- a mass condition, an assembly
condition, and a proposition for the case that no budget repairs --- turned out to be at least
as informative as the effect estimates. On retrieval, the observed benefit is
dominated in point estimates by within-support selection; redistribution is not resolved at the
available power, which is not the same as being zero. On mathematics the honest output of the audit is not a number but a
structure: extinction leaves the ordinary comparison undefined at every budget, while the extended
comparison is constructible only inside narrow budget bands that the operating point we had
pre-registered happens to miss. Four of five text-to-SQL pools have empty feasible regions and the
fifth has a narrow band. None of that is a number a practitioner can act on, and all of it is
something they need before trusting one.

\section*{Availability}

The implementation is released at \url{https://github.com/ThakiCloud/cgsd} under Apache-2.0,
together with the recorded result files and the scripts that generate every figure, table and
number in this paper from them. Nothing here is transcribed by hand, so a reader can rerun that
pipeline and reproduce the reported values exactly. The constructor, both feasibility tests and
the extinction split ship with the data-free gate harnesses used to pin them, which execute
without any corpus. What is not released is the data: the source pools and evaluation splits are
not ours to redistribute, so retraining the arms requires equivalent data of the reader's own.
The estimand, the decomposition identity, both feasibility conditions, the extinction
proposition and the drift bound are fully specified in
Sections~\ref{sec:estimand}--\ref{sec:extinction} and can be reimplemented from the text alone.

\bibliographystyle{plainnat}
\bibliography{refs}

\end{document}